\documentclass[conference]{IEEEtran}
\IEEEoverridecommandlockouts
\usepackage{cite}
\usepackage{amsmath,amssymb,amsfonts,amsthm}

\usepackage{bbm}
\usepackage{algorithm}
\usepackage{algpseudocode}
\usepackage{graphicx}
\usepackage{graphics}
\usepackage{textcomp}
\usepackage{xcolor}
\usepackage{multirow}
\usepackage{colortbl}
\usepackage{url}
\usepackage{booktabs}
\usepackage{hyperref}
\usepackage[normalem]{ulem}
\useunder{\uline}{\ul}{}
\def\BibTeX{{\rm B\kern-.05em{\sc i\kern-.025em b}\kern-.08em
    T\kern-.1667em\lower.7ex\hbox{E}\kern-.125emX}}

\newtheorem{theorem}{Theorem}

\newcommand{\vecX}{\mathbf{x}}
\newcommand{\vecY}{\mathbf{y}}

\newcommand{\vecT}{\boldsymbol{\theta}}
\newcommand{\matD}{\mathbf{D}}
\newcommand{\matG}{\mathbf{G}}
\newcommand{\matX}{\mathbf{X}}
\newcommand{\matZ}{\mathbf{Z}}
\newcommand{\spaceG}{\mathcal{G}_{\mathrm{rsto}}}

\theoremstyle{definition}

\begin{document}

\title{A theoretical framework for self-supervised contrastive learning for continuous dependent data\\
% \thanks{Identify applicable funding agency here. If none, delete this. Do we have a funding agency?}
}

% \author{\IEEEauthorblockN{Anonymous Authors}}

\author{\IEEEauthorblockN{1\textsuperscript{st} Alexander Marusov}
\IEEEauthorblockA{\textit{Applied AI Center} \\
\textit{Skoltech}\\
Moscow, Russia\\ 
A.Marusov@skoltech.ru}
\and
\IEEEauthorblockN{2\textsuperscript{nd} Aleksandr Yugay}
\IEEEauthorblockA{\textit{Applied AI Center} \\
\textit{Skoltech}\\
\textit{MIPT}\\
Moscow, Russia\\
iugai.aa@phystech.edu}
\and
\IEEEauthorblockN{3\textsuperscript{rd} Alexey Zaytsev}
\IEEEauthorblockA{\textit{Applied AI Center} \\
\textit{Skoltech}\\
\textit{Risk Management} \\
\textit{Sber}\\
Moscow, Russia\\
A.Zaytsev@skoltech.ru}
}
\maketitle

\begin{abstract}
Self-supervised learning (SSL) has emerged as a powerful approach to learning representations, particularly in the field of computer vision. 
However, its application to dependent data, such as temporal and spatio-temporal domains, remains underexplored. 
Besides, traditional contrastive SSL methods often assume \emph{semantic independence between samples}, which does not hold for dependent data exhibiting complex correlations. 
We propose a novel theoretical framework for contrastive SSL tailored to \emph{continuous dependent data}, which allows the nearest samples to be semantically close to each other. In particular, we propose two possible \textit{ground truth similarity measures} between objects --- \emph{hard} and \emph{soft} closeness. Under it, we derive an analytical form for the \textit{estimated similarity matrix} that accommodates both types of closeness between samples, thereby introducing dependency-aware loss functions.
We validate our approach, \emph{Dependent TS2Vec}, on temporal and spatio-temporal downstream problems. 
Given the dependency patterns presented in the data, our approach surpasses modern ones for dependent data, highlighting the effectiveness of our theoretically grounded loss functions for SSL in capturing spatio-temporal dependencies.
Specifically, we outperform TS2Vec on the standard UEA and UCR benchmarks, with accuracy improvements of $4.17$\% and $2.08$\%, respectively. Furthermore, on the drought classification task, which involves complex spatio-temporal patterns, our method achieves a $7$\% higher ROC-AUC score.

\end{abstract}


\section{Introduction}
Self-supervised learning (SSL) methods demonstrate significant potential in achieving the quality of supervised models using a vast amount of unlabeled data \cite{nguyen2024improving, shang2024transitivity, Marusov23, bazarova2024normalizing}. 
The SSL paradigm provides a method for training an encoder using a vast amount of \emph{unlabeled} data, which produces a \emph{universal} representation (embedding) of an input signal. 
Since representations are universal, we can use them to solve various downstream problem tasks without time- and data-consuming full retraining of the encoder~\cite{EMIT, Romanenkova22}. 

The core issue with SSL approaches is the existence of the \textit{complete collapse} problem, i.e., the situation where all objects have the same representation~\cite{Jing22}. The contrastive SSL algorithms propose a natural solution to such a problem --- they pull representations of semantically similar objects (\textit{positive pairs}) closer in the embedding space while pushing apart representations of dissimilar ones (\textit{negative pairs})~\cite{Joshi23}.

SSL algorithms for computer vision domain, including the most prominent ones like SimCLR~\cite{Chen20}, MoCo~\cite{Chen21}, BYOL~\cite{Gril20}, Barlow Twins~\cite{Zbontar21}, VICReg~\cite{bardes2022vicreg}, DINO~\cite{Caron21}, DinoV2~\cite{oquab2023dinov2}, utilize \textit{semantic independence} assumption, which states that all objects belong to different classes. Thus, to construct a positive pair, researchers typically apply data augmentation techniques (e.g., random cropping, color jittering, or noise injection) to a single instance, thereby generating semantically similar views of the same underlying object.
All other pairs of objects are treated as negatives.

In the dependent data domain (e.g., time series), we should consider possible correlations between samples, which lead to more accurate definitions of positive and negative pairs. Thus, the hypothesis of \textit{semantic independence for distinct objects} is invalid for such a case. However, in the general case, we can't believe that closely lying samples form a positive pair due to probable anomalies in the data~\cite{Yue22}. 
Thus, most of the existing works (e.g., TS2Vec~\cite{Yue22}, CoST~\cite{woo2022cost}, TimesURL~\cite{liu2024timesurl}), on the one hand, try to account for the features that are specific for time series but on the other --- continue to use standard loss functions from computer vision domain, designed with a semantic independence assumption in mind. Instead of a straightforward usage of the semantic independence hypothesis, modern approaches (e.g., Soft-contrastive learning~\cite{lee2024soft}) try to define closeness between objects in a softer way. However, these methods rely more on heuristics rather than creating theoretically based similarities between samples. 
% For example, TS2Vec considers different objects from nearby timestamps as a negative pair.
% Soft-contrastive learning, SoftCLT~\cite{lee2024soft}, rather than considering samples as semantically independent, defines a smoother closeness between objects, resulting in improved performance in empirical studies. 
% However, the proposed dependence structure between samples covers only a limited class, even with a family of stationary time series~\cite{jenkins1976time}. 

We aim to obtain theoretically based loss functions that account for a diverse range of possible correlations in dependent data. However, as we have already mentioned, temporal consistency is typically absent in time series. Thus, we propose a single central assumption on \textit{data continuity}, i.e., closely lying objects can be viewed as a positive pair. Under it, we propose a theoretical framework for the SSL for continuous dependent data. In particular, we propose several ways of considering different possible similarities between samples (\textit{hard and soft dependency}), which are contained in corresponding \textit{similarity matrices}. The intuition of our approach is depicted in Fig.~\ref{fig:framework}. As a result, we obtained several loss functions that are suitable for different types of similarity between elements. We verified the effectiveness of our framework by considering different temporal and spatio-temporal downstream problems. %In particular, we considered climate-related data that satisfies our data continuity assumption because the samples are known to be highly correlated and anomaly-free~\cite {Wang2022}.

\begin{figure}[t]
    \centering
    \includegraphics[scale=0.3]{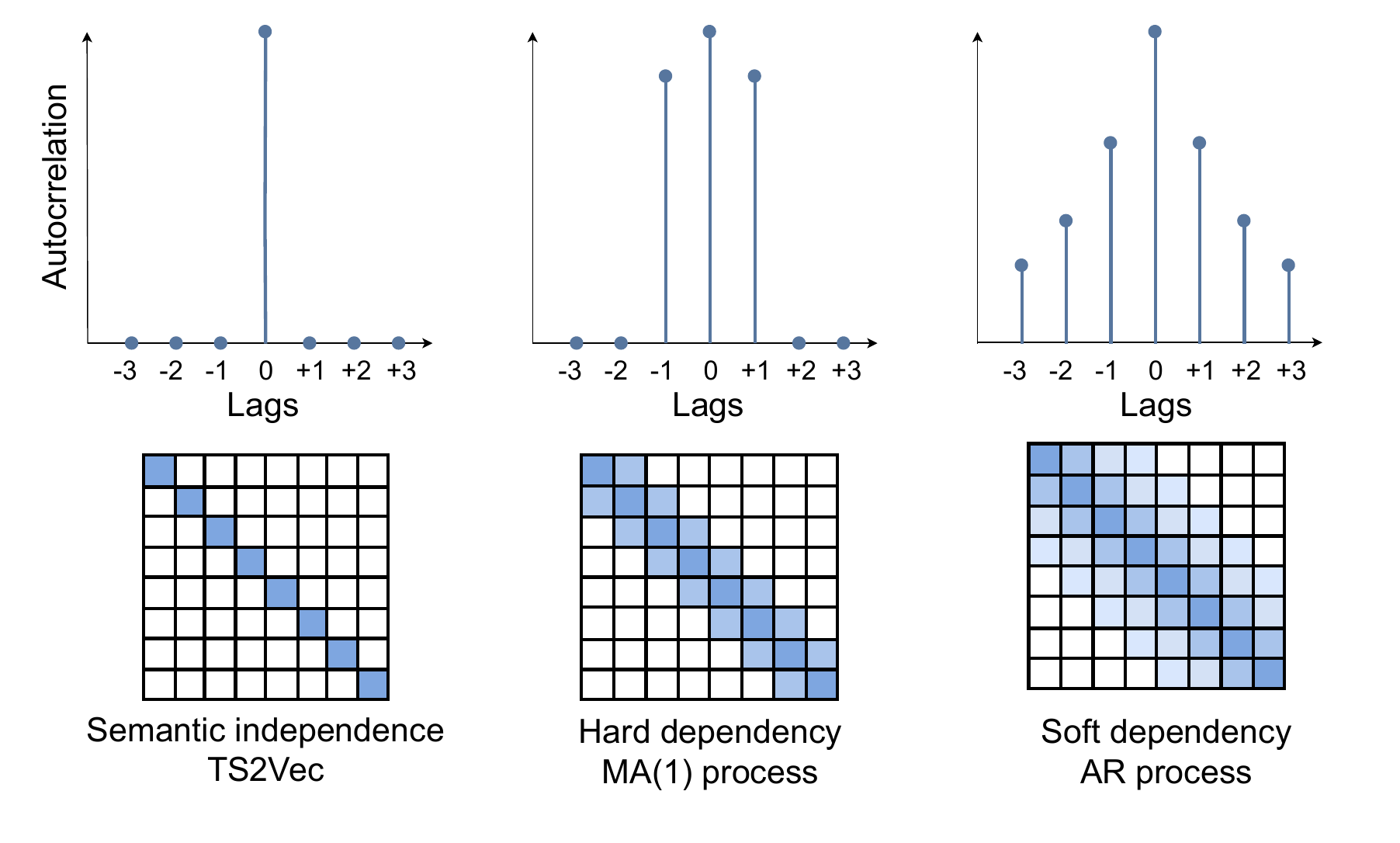}
    \caption{The autocorrelation function (ACF, top row) visualizes the dependency between objects in time series processing, and the correlation matrices (bottom row) are also a standard way to plot similarities between elements. For the ACF, higher values correspond to bigger correlations; for the matrices, the darker colors correspond to higher correlations.
    We provide plots for three different possible semantic connections between samples: semantic independence and two of our hard and soft dependencies. In the first case, all samples are not semantically related to each other. For \emph{hard dependency}, similarly to an MA(1) process, the adjacent samples are semantically connected, while others are not. For \emph{soft dependency}, similarly to an AR process, the closeness between elements decreases exponentially with the distance between them.}
    \label{fig:framework}
\end{figure}

Our implementation is publicly available at our GitHub\footnote{\url https://github.com/aayugay99/DepTS2Vec}.
%\href{https://github.com/aayugay99/DepTS2Vec}{https://github.com/aayugay99/DepTS2Vec}.
The key novelties and contributions of this work are as follows:
\begin{itemize}    
    \item \textbf{SSL theoretical framework for continuous dependent data.} Current theoretical fundamental studies~\cite{Balestriero22, Jing22, Tian21, ji2023power} for the discriminative SSL assume semantic independence between samples. Based on the data continuity assumption, we propose a theoretical SSL framework for dependent data that considers different closeness types between objects — \textit{hard and soft dependency}. 
    The information about the closeness between samples for each dependency type is collected in corresponding \emph{ground truth similarity matrix}. In this way, our hard and soft dependency relations correspondingly mimic moving average (MA) and autoregression (AR) components that together provide a complete decomposition of a stationary time series, ARIMA~\cite{jenkins1976time}, as the Wold theorem states~\cite{wold1938study}. 
    
    % To account for different pairwise similarities, we considered corresponding dependency types --- \emph{nearby} (only closest elements are positive) and \emph{exponential} (positive non-binary soft-relationship, i.e., the impact of each element to the target one is inversely proportional to the distance between them). In this way, our framework mimics moving average and autoregression components that together provide a complete decomposition of a stationary time series, ARIMA~\cite{jenkins1976time}, as the Wold theorem states~\cite{wold1938study}. According to our results, each relation type leads to a specific loss function that should converge to a reasonable non-degenerative solution.

    \item \textbf{Optimal closed-form estimated similarity matrix for continuous dependent data.}
    To estimate the similarity matrix, we adapted an optimization problem originally developed for semantically independent data to one suitable for continuous, dependent data. By solving this problem we obtained closed-form \emph{estimated similarity matrix} for continuous dependent data. We prove that this matrix is a local minimum of the corresponding optimization problem.

    \item \textbf{Practical application.} 
    For each of dependency types we constructed own loss functions using proposed \emph{ground truth and estimated similarity matrices.}To validate our theoretical framework in practice, we introduce \emph{Dependent TS2Vec}, a method that incorporates our dependency-aware loss function into the standard TS2Vec architecture.
    % To validate our findings, we consider different downstream problems. Considering climate-related data, which is mostly suitable for the data continuity assumption we can see the superiority of our method over the other approaches.
    On UCR and UEA time series classification benchmarks, our method outperforms TS2Vec by 2.08\% and 4.17\% in accuracy respectively. For drought classification, where spatio-temporal continuity is a key characteristic, we achieve a 7\% improvement in ROC-AUC score.
    
    % the drought prediction task for regions from different climate zones. As a baseline, the TS2Vec model was generalized from time series case to spatio-temporal. 
    % Our spatio-temporal TS2Vec with the proposed loss function performs better than the original one.
\end{itemize}

\section{Related works}
\label{sec:rw}

% TODO: we need an intro here about how the structure of this section emerged
\subsection{Time series SSL}
TS2Vec~\cite{Yue22} positions itself as a universal SSL-approach. To create positive pairs, researchers take intersecting time intervals and make representations of the same timestamp closer, considering all other samples as negative. CoST framework~\cite{woo2022cost} splits the original time series into trend and seasonal components and considers the transition to the frequency domain. TimesURL~\cite{liu2024timesurl} expands the idea of CoST and TS2Vec by considering augmentations from the frequency domain, using the objective function from TS2Vec, and adding an extra loss component that estimates the error of reconstructing the initial signal from the masked one. A modern soft-contrastive learning approach, SoftCLT~\cite{lee2024soft}, rather than considering samples as semantically independent, defines a smoother closeness between objects, resulting in improved performance in empirical studies. 
However, the proposed dependence structure between samples covers only a limited class, even with a family of stationary time series~\cite{jenkins1976time}. Given the existing body of work, we observe that current methods either treat time series samples as semantically independent or propose possible heuristic similarity relations between elements without theoretical justification. 

\subsection{Spatio-temporal SSL}
\label{related work: spatio-temporal ssl}
% SSL methods for \textbf{spatio-temporal} domain show promising results. Masked autoencoder~\cite{Feichtenhofer22} applied spatio-temporal masking of different patches in a video and tried reconstructing an initial signal. 
% To extract spatio-temporal correlations from graph-based data, authors~\cite{Wang23} used BYOL architecture. For each node, they apply random masking for the signal to create a \textit{temporal} positive pair and assume the neighbor vertex as a \textit{spatial} positive pair.

Spatio-temporal data can be divided into three major subdomains: video representation, traffic flow prediction, and climate forecasting~\cite{tan2023openstl}. Below, we consider the application of SSL for each scope in more detail.

\paragraph{Video representation} Generative approaches are widely used to obtain video representations. Authors~\cite{Feichtenhofer22} used a masked autoencoder to receive informative representations by reconstructing an initial signal from the masked one. Researchers~\cite{donahue2024videossl} learned different video-specific features, such as redundant actions and background frames, using the proposed regression loss function. To enhance the quality of the action prediction task, the authors proposed a temporal DINO architecture that captures both temporal and spatial dependencies~\cite{teeti2023tempdino}.

\paragraph{Traffic flow prediction} The authors in~\cite{Wang23} represent the data using a graph-based approach, where each node corresponds to the traffic flow at a specific station. To extract spatio-temporal correlations, researchers utilized BYOL architecture and specifically defined positive pairs. In particular, \textit{spatial positive pair} assumes neighboring vertexes, while to create \textit{temporal positive pair}, the authors randomly mask the initial signal. Instead of a standard random masking, researchers~\cite{gao2024traffic} proposed a traffic-specific augmentation approach based on the similarity between spatial regions. Another recent approach~\cite{li2025STMFormer} utilized a transformer, which used spatial and temporal attention modules to extract corresponding dependencies from the data.

\paragraph{Climate forecasting} Climate prediction holds a distinctive position among all spatio-temporal challenges. The standard contrastive architecture, SimCLR, was used for weather station classification in East China~\cite{Wang2022}. The authors~\cite{hoffmann2023atmodist} proposed to use siamese network~\cite{chicco2021siamese} to obtain informative representations of atmospheric fields. To extract universal embeddings of the multispectral and synthetic aperture radar images, researchers~\cite{nawaz2025restricted} proposed a self-supervised learning framework that doesn't require any negative labels. 

In our research, we focused on climate-related tasks for several reasons:
\begin{itemize}
    \item Weather data tends to be relatively stable over time~\cite{Wang2022}. The stability of climate data is essential, as our theoretical SSL framework requires the assumption of data continuity, i.e., when at least nearby samples form a positive pair.

    \item Climate-related tasks play a pivotal role in human life. For example, drought prediction modeling is vital in the context of global climate change~\cite {Xiujia22}.
\end{itemize}

\subsection{SSL theoretical background}

While in practice, researchers have found different solutions to prevent a complete collapse problem, they still need a theoretical justification for the proposed methods. While contrastive-based approaches naturally prevent \textit{complete collapse}, i.e., when embeddings for all samples are identical, it remains unclear how non-contrastive methods address this challenge. Under a linear setting, i.e., when the encoder is a linear model, the authors~\cite{Tian21} have studied in detail one of the non-contrastive methods --- Bootstrap Your Own Latent, BYOL~\cite{Gril20}. Another fundamental study was dedicated to the \textit{dimensional collapse} problem --- when representations of different objects lie in a subspace of a low dimension. Researchers~\cite{Jing22} found an intriguing property of SimCLR: despite preventing a complete collapse, it still suffers from dimensional collapse. To understand the nature of contrastive learning, the authors~\cite{Balestriero22} estimated a similarity matrix for semantically independent samples (see Section~\ref{method:hard_contrastive_learning} for much more details).

\subsection{Research gap}
\label{related work: research_gap}
According to our review, we see the following gaps:
\begin{itemize}
    %\item There are only a few works on spatio-temporal SSL
    \item SSL approaches for temporal and spatio-temporal data heavily rely on the loss function, which is designed for the case with semantically independent samples. However, there exists a class of applied problems (e.g., a wide range of climate modeling tasks) where the samples are semantically related to each other. Thus, such tasks need a specific objective function that accounts for data peculiarities.
    
    \item Despite the existence of the SSL theory for semantically independent data, it is still absent for dependent data. 
\end{itemize}

\section{Method}
\label{sec:method}
\subsection{Unsupervised contrastive learning problem statement}

\paragraph{Notation}
Below, we introduce the notation that will be used for the problem definition, mostly following~\cite{Balestriero22}:
\begin{itemize}
    %\item $\mathbbm{R}_{\ne 0} =  \mathbbm{R} \setminus \{0\}$ - set of real numbers excluding zero.
    \item $\vecX \in \mathbb{R}^k$ --- an input, a sample of inputs $\matX = \{\vecX_i\}_{i = 1}^N$. 
    \item $f_{\vecT}(\vecX) \in \mathbb{R}^{K}$ --- embedding of an input $\vecX$ produced by an encoder $f_{\vecT}$ with parameters $\vecT \in \mathbb{R}^p$.
    \item $\mathbf{Z} = \{f_{\vecT}(\vecX_i) \}_{i = 1}^N \in \mathbb{R}^{N \times K}$ --- a matrix of embeddings $f_{\vecT}(\vecX)$.
    \item $d_{ij} = d(f_{\vecT}(\vecX_i), f_{\vecT}(\vecX_j))$ --- a distance (e.g. cosine) between embeddings $f_{\vecT}(\vecX_i)$ and~$f_{\vecT}(\vecX_j)$.
    \item $\matD = \{d_{ij}\}_{i, j = 1}^N \in \mathbb{R}^{N \times N}$ --- a matrix of pairwise distances~$d_{ij}$.
    \item $\matG = \{g_{ij}\}_{i, j = 1}^N \in \mathbb{R}^{N \times N}$ --- a ground truth similarity matrix between samples.

    \item $\widehat{\matG} = \{\hat{g}_{ij}\}_{i, j = 1}^N \in \mathbb{R}^{N \times N}$ --- the  estimated similarity matrix $\widehat{\matG} \triangleq \widehat{\matG}(\matZ)$, based on $\mathbf{Z}$, which we want to make close to $\matG$. As shown below, it differs from $\matD$.
\end{itemize}

Additional notation introduces spaces of matrices and regularizations:
\begin{itemize}
    \item $\mathcal{G}$ --- a set of symmetrical matrices with non-negative values and zero diagonal.
    \item $\mathcal{G}_{\mathrm{rsto}} = \{ \matG \in \mathcal{G}: \matG \textbf{1} = \textbf{1} \}$ --- a space of right-stochastic matrices.
    \item $\mathcal{R_{\log}}(\matG) = \tau \sum_{i \neq j} g_{ij} (\ln(g_{ij}) - 1)$, where $\tau$ is a temperature constant, is the \emph{Log regularization}.
    %\item $\mathcal{R_{F}}(\matG) = \tau \sum_{i \neq j}  \mathbf{G}_{ij} (\mathbf{G}_{ij}/2 - 1)$ is \emph{the $\mathcal{F}$-regularization}.
\end{itemize}

\paragraph{Method pipeline}

In representation learning, the objective is to fit an encoder $f_{\boldsymbol{\theta}}(\vecX)$ that generates meaningful representations for each sample $\vecX$~\cite{bengio2013representation}. When labeled data is limited, self-supervised learning (SSL) is utilized to train an encoder, relying solely on unlabeled samples $X$. 

Contrastive self-supervised learning, a type of SSL, relies on the proximity between objects, determined without supervision. Specifically, we identify which samples are \textit{similar} to each other and which are \textit{different}. This information is typically captured in a ground truth similarity matrix $\matG$. The primary goal of the SSL is to develop an encoder that produces representations reflecting these relationships. More formally we need to build such encoder $f_{\vecT}: \forall x, y \in \matX \rightarrow sim(x,y) \approx sim(f_{\vecT}(\vecX), f_{\vecT}(\vecY))$, where $sim(x, y)$ --- is a ground truth similarity measure between objects $x$ and $y$.

To implement the idea above, we need to estimate a similarity matrix $\widehat{\matG}(\matZ)$ based on embeddings $f_{\boldsymbol{\theta}}(\vecX)$ and optimize parameters $\boldsymbol{\theta}$ to produce $\widehat{\mathbf{G}}$ such that it is close to the ground truth matrix $\mathbf{G}$~\cite{Balestriero22}:
\begin{equation}
\label{eq:loss_sim_matrix}
\mathcal{L}(\vecT) = - \sum_{i,j} g_{ij} \ln  \widehat{g}_{ij}(f_{\boldsymbol{\theta}})  \rightarrow \min_{\vecT},
\end{equation}
%where $\widehat{\matG} = \widehat{\matG}(\mathbf{Z}(f_{\vecT}(\matX)))$.

In the image domain, the corresponding matrix $\matG$ is based on the assumption that two distinct images $\vecX_i$ and $\vecX_j$ are semantically independent: $g_{ij} = 0$ for them. Thus, to build a positive pair for any image $\vecX_{i}$ we create its augmented version $\vecX_{i^{'}}$, and for these two objects $g_{ii^{'}} = 1$. 
We refer to this approach as \emph{Semantic independence} since all samples are semantically independent of each other.

For dependent data, such as time series, the assumption of semantic independence generally does not hold due to non-zero correlations between samples.
Therefore, we require a different method that takes into account the similarity between closely related samples.

To create a correct substitute loss function~\eqref{eq:loss_sim_matrix} for dependent data, we need to answer the following questions:
\begin{itemize}
    \item[\textbf{Q1}] How does the ground truth similarity matrix $\matG$ should be defined for dependent data?
    \item[\textbf{Q2}] How can we estimate $\widehat{\matG}$ based on $\matZ$ to effectively capture the dependencies in $\matG$?
\end{itemize}
 
To address the first question, we suggest two approaches for constructing a ground truth similarity matrix $\matG$ for dependent data, which consider the closeness between samples: --- \emph{hard} and \emph{soft} dependency relations. 
For the second question, we adapted the optimization problem originally designed for independent samples to handle dependent data by modifying the constraints on the ground truth similarity matrix $\matG$. 
Solving these optimization problems for different types of sample relationships (\emph{hard} and \emph{soft} dependency) results in the estimated matrix $\widehat{\matG}$. 

Using $\matG$ and $\widehat{\matG}(\matZ)$, we formulate a loss function based on equation \eqref{eq:loss_sim_matrix} and evaluate its performance on different downstream problems. 
In the section below, we start with the formulation of the optimization problem for the \textit{Semantic independence} case.

\subsection{Semantic independence}
\label{method:hard_contrastive_learning}

In the image domain, samples are typically semantically independent. To generate a positive example for an image, various augmentation techniques, such as Random Crop and Gaussian Blur, are applied~\cite{Chen20}. Consequently, each sample is paired with one positive example, while all other samples are considered negatives. Therefore, the authors~\cite{Balestriero22} utilized right-stochastic similarity matrices $\spaceG$, where each row sums to $1$. Additionally, they set the diagonal elements to zero to align the SimCLR similarity matrix with Laplacian estimation.

Formally, to recover the similarity matrix $\hat{\matG}$, the following optimization problem is defined:
\begin{equation}
\label{opt:prob}
\begin{aligned}
\mathrm{Tr}(\matD \matG) &+ \mathcal{R}(\matG) \rightarrow \min_{\matG \in \spaceG}, \\
\textrm{s.t. } %\matG_{ii} &= 0 \qquad \text{for any\,\,} i \in \overline{\{1, N\}}, \\
\sum_{j\neq i}g_{ij} &= 1 \qquad \text{for any\,\,} i \in \overline{\{1, N\}}, \\
\end{aligned}
\end{equation}

where $\mathcal{R}$ is a regularizer.
Using $\mathcal{R_{\log}}$ regularizer, authors get SimCLR similarity matrix $\widehat{\matG}$ with elements:
\begin{equation}
\label{eq:sim_orig}
\begin{aligned}
\widehat{g}_{ij} &=& \frac{e^{- \frac{1}{\tau}d(f_{\vecT}(\vecX_i), f_{\vecT}(\vecX_j))} }{\sum_{j \neq i} e^{-\frac{1}{\tau}d(f_{\vecT}(\vecX_i), f_{\vecT}(\vecX_j))}} \mathbbm{1}_{\{i \neq j \}}.
\end{aligned}
\end{equation}
Now we can plug this matrix in~\eqref{eq:loss_sim_matrix}
and solve the optimization problem w.r.t. the encoder parameters $\vecT$.

\subsection{Our approach}
\label{method:our_approaches}

As we mentioned earlier, to construct an objective function~\eqref{eq:loss_sim_matrix}, we need to define both the \emph{ground truth and estimated similarity matrices}. In the section below, we define different possible ground truth similarity matrices $\matG$. After that, we formulate an optimization problem to recover an estimated similarity matrix $\widehat{\matG}$. Lastly, we put together the $\matG$ and $\widehat{\matG}$ matrices to create a loss function suitable for \emph{continuous dependent data}. 

\subsubsection{Ground truth similarity matrices}
The estimated similarity matrix~\eqref{eq:sim_orig} assumes semantic independence between samples, which is not always valid. For time series and spatio-temporal data, samples typically exhibit non-zero correlations.

To address these dependencies, we seek to develop more appropriate similarity matrices. %Under the stationarity assumption for time series, correlations between observations can be modeled using either a limited number of non-zero correlations or an exponential decay~\cite{box2015time}.

As a result, we consider two types of relationships between elements:
\begin{itemize}
\item \emph{Hard dependency}. Only closely located observations are correlated.
\item \emph{Soft dependency}. The correlation decreases steadily as the distance between elements increases.
\end{itemize}

% Wold theorem~\cite{wold1938study} specifies that stationary time series can be decomposed into moving average (MA) and autoregression (AR) terms, leading to a justification for the famous ARIMA model~\cite{jenkins1976time}. Our \emph{nearby} and \emph{exponential} relations model moving average with a unit lag $\mathrm{MA(1)}$ and autoregression $\mathrm{AR}$ components correspondingly. Indeed, the autocovariance function of $MA(1)$ process truncates to zero after lag $k = 1$ (i.e., nonzero covariation only with the nearest object, see ~\cref{fig:framework})  ~\cite{brockwell2002introduction}. In contrast, the AR process has nonzero components for any lag. The idea of exponential relation also exists as a soft-relationship graph for self-supervised learning in the computer vision domain~\cite{Sobal24}.

%Each type of dependency can be described with corresponding similarity matrix. Since our similarity matrices are contrastive-based representations do not suffer from complete collapse. However, they may still suffer from dimensional collapse like it was shown in ~\cite{Jing22}. As we will show in our experiments (see Section~\ref{practice}), the best-performing similarity matrix is in~~\cref{thm:nearby_log_reg}.

%In practice, we observe that for the considered applied problem, the best-performing model corresponds to $\mathrm{MA(1)}$ form of dependencies from ~\cref{thm:nearby_log_reg}.

To formalize two cases of dependencies, we consider a sequential time series $\vecX_1, \vecX_2, \vecX_3, \ldots, \vecX_N$ that correspond to sequential moments in time from $1$ to $N$. Below, we introduce ground-truth similarity matrices for hard and soft dependencies, respectively.

\textbf{Hard dependency (moving average).}
Here, neighbor observations $\vecX_i$ and $\vecX_j$ with close $i$ and $j$ such that $|i - j| = 1$ are dependent.
Thus, closely lying samples can be treated as positive ones. 
Then, we can introduce the following model for the matrix $\matG$:

\begin{equation}
\label{sim:ma_sim_mat}
g_{ij} =
\begin{cases}
  1, & \text{if}\ |i - j| = 1, \\
  0, & \text{otherwise}.
\end{cases}
\end{equation}

\textbf{Soft dependency (autoregression).}
In the long-range case the closeness between elements decreases steadily and can be formalized in the following way. Let $a_{ij} = \exp\left(-\frac{(i-j)^2}{k}\right)$, where $k > 0$.

Then 
\begin{equation}
\label{sim:ar_sim_mat}
g_{ij} =
\begin{cases}
  \frac{a_{ij}}{\sum_{l = \min(i,j) + 1}^{N} a_{\min(i, j) l}} & \text{if}\ i \neq j, \\
  0, & \text{otherwise}.
\end{cases}
\end{equation}

% In the exponential case, the closeness between elements can be formalized in the following way:
% \begin{equation}
% \label{exp:matrix}
% \matG_{ij} =
% \begin{cases}
%   \exp \left(-\frac{(i-j)^2}{k} \right), & \text{if}\ i \neq j, \\
%   0, & \text{otherwise}.
% \end{cases}
% \end{equation}
% with $k > 0$. Before we can formulate the corresponding optimization problem, we need to find the constraint for this variant of a matrix.

% \begin{lemma}
% Let $\matG$ be a similarity matrix defined by ~\eqref{exp:matrix}. Then 
% \label{lem:constraint}
% \begin{equation}
% \sum_{i \neq j} \ln(\matG_{ij}) = \frac{N^2 (1 - N^2)}{6k}.
% \end{equation}

% where $k > 0$.

% \end{lemma}

% \begin{proof}
%     See proof in Appendix~\ref{app:proof_lemma}.
% \end{proof}

% Similarly to the nearby relation we also proposed the constraint from Lemma~\ref{lem:constraint} to account for specific properties of the $\matG$.
% Hence, the optimization problem for the exponential relation case is the following:
% \begin{equation}
% \label{sim:exp}
% \begin{aligned}
% \widehat{\mathbf{G}}_{d, \mathcal{R}} &= \argmin_{\matG \in \mathcal{G}} \mathrm{Tr}(\matD \matG) + \mathcal{R}(\matG), \\
% \textrm{s.t.} \sum_{i \neq j} \ln(\matG_{ij}) &= \frac{N^2 (1 - N^2)}{6k},  k > 0,\\
% %\matG_{ii} &= 0 \qquad \text{for any\,\,} i \in \overline{\{1, N\}}. \\
% \end{aligned}
% \end{equation}

\subsubsection{Estimated similarity matrix}

In the previous section, we defined different possible ground truth similarity matrices. To construct an objective function~\eqref{eq:loss_sim_matrix} we need to estimate a similarity matrix based on the embeddings. Below, we formulate an optimization problem to recover $\widehat{\matG}$.

\paragraph{Optimization problem}
Even though different ground truth similarity matrices describe hard and soft dependency relations between samples, the constraints on them are the same. This fact allows us to introduce the optimization problem suitable simultaneously for both possible dependency types between samples:
\begin{equation}
\label{opt_prob}
\begin{aligned}
\mathrm{Tr}(\matD \matG) &+ \mathcal{R}(\matG) \rightarrow \min_{\matG \in \spaceG},\\
\textrm{s.t.}
\sum_{j = i + 1}^{N} 
g_{ij} &= 1 \qquad \text{for any\,\,} i \in \overline{\{1, N\}}, \\
\sum_{i = j + 1}^{N} g_{ij} &= 1 \qquad \text{for any\,\,} j \in \overline{\{1, N\}}. \\
%\matG_{ii} = 0 \qquad \text{for any\,\,} i \in \overline{\{1, N\}}.
\end{aligned}
\end{equation}

\paragraph{Necessary condition}
Given the optimization problem with constraints~\eqref{opt_prob} for continuous dependent data, we provide analytical solution for the problem, defining appropriate $\hat{\matG}(\matZ)$. In particular, we show the necessary condition of the minimum. 

% \paragraph{Short-range relation.}
% For two appropriate regularizers, we obtain the analytical form for $\hat{\matG}$ below.

\begin{theorem}
\label{thm:nearby_log_reg}
Solving the optimization problem \eqref{opt_prob} using $\mathcal{R_{\log}}$ regularization gives the following estimation of the similarity matrix:
\begin{align*}
(\widehat{g}_{d, \mathcal{R_{\log}}})_{ij} = \frac{e^{- \frac{1}{\tau} d_f(\vecX_i, \vecX_j)} }{\sum_{k = \min(i, j) + 1}^{N} e^{-\frac{1}{\tau}d_f(\vecX_{\min(i, j)}, \vecX_k)}} \mathbbm{1}_{\{i \neq j \}},
\end{align*}
where $d_f(\vecX, \vecX') = d(f_{\theta}(\vecX), f_{\theta}(\vecX'))$.

% (\widehat{\mathbf{G}}_{d, \mathcal{R_{\log}})_{ij}} = \frac{e^{- \frac{1}{\tau}d(f_{\theta}(\vecX_i), f_{\theta}(\vecX_j))} }{\sum_{k = \min(i, j) + 1}^{N} e^{-\frac{1}{\tau}d(f_{\theta}(\vecX_{\min(i, j)}), f_{\theta}(\vecX_k))}} \mathbbm{1}_{\{i \neq j \}}.

\begin{proof}
    See proof in Appendix \ref{app:proof_nearby_log_reg}.
\end{proof}
\end{theorem}

% \paragraph{Long-range relation.}
% We obtain the analytical form for $\hat{\matG}$ below.

% \begin{theorem}
% \label{thm:exp}
%  Solving the optimization problem~\eqref{sim:exp} using contsraint from~\eqref{lem:constraint} without regularizer gives the following estimation of the similarity matrix:

%  % \mathbbm{1}_{\{i \neq j \}}
% $$
% %(\widehat{\mathbf{G}}_{d_{ij}}) = \frac{\exp(\frac{1}{N^2 - N} (C + \sum_{i \neq j}\ln\left(d(f_\theta(x_i), f_\theta(x_j))\right)}{d(f_\theta(x_i), f_\theta(x_j))}\mathbbm{1}_{\{i \neq j \}}
% (\widehat{\mathbf{G}}_{d})_{ij} = 
% \begin{cases}
%     \frac{\exp\left(\frac{1}{N^2 - N} (C + \sum_{i \neq j}\ln\left(d(f_\theta(\vecX_i), f_\theta(\vecX_j))\right) \right)}{d(f_\theta(\vecX_i), f_\theta(\vecX_j))}, &i\neq j,\\
%     0, &i=j,
% \end{cases}
% $$
% where
% $$
% C = \frac{N^2 (1 - N^2)}{6k}, k > 0.
% $$
% \textit{Note:} constant $C$  was obtained in Lemma~\ref{lem:constraint}.
% \end{theorem}
% \begin{proof}
%     The proof is completely similar to the one given in Appendix \ref{app:proof_nearby_log_reg}.
% \end{proof}

\paragraph{Sufficient condition}
The theorem above represents the necessary condition for the extremum. However, it's not sufficient. Below, we will demonstrate that the obtained similarity matrix indeed yields the minimum. The key point of such a nice property is the presence of the regularizer. To facilitate understanding of the last property, we provide the proof of sufficiency for the estimated similarity matrix directly in the text.

\begin{theorem}
\label{thm:suf_nearby_log_reg}
 The estimated similarity matrix from Theorem~\ref{thm:nearby_log_reg} is a unique minimum of the optimization problem~\eqref{opt_prob} with Logarithmic regularizer. 
\end{theorem}
\begin{proof}
\label{proof:suf_nearby_log_reg}
To check the sufficient condition of the local extremum, we need to calculate the Hessian of the Lagrange function, which is expressed in the following way:
\begin{align*}
\mathcal{L} = \sum_{i \neq j}d(f_{\theta}(\vecX_i), f_{\theta}(\vecX_j)) g_{ij} + \tau \sum_{i \neq j}  g_{ij}\left(\ln(g_{ij}) - 1 \right) + \\ \sum_{i=1}^{N} \alpha_{i} \left(\sum_{j=i+1}^{N}g_{i,j} - 1 \right) + \sum_{j=1}^{N} \beta_{j} \left(\sum_{i=j+1}^{N}g_{i,j} - 1 \right).% + \sum_{i} \gamma_{i} \mathbf{G}_{i,i},
\end{align*}

By definition, the Hessian of the Lagrange function is 
$$
d^2\mathcal{L} = \sum_{i\neq j,i^{'}\neq j^{'}} \frac{\partial^{2}{\mathcal{L}}}{\partial{g_{ij}}\partial{g_{i^{'}j^{'}}}} dg_{ij} dg_{i^{'}j^{'}}.
$$

To check the positivity of this quadratic form (Hessian), we need to calculate the second-order derivatives. To do this, let's recall the expression of the first partial derivative:
$$
\frac{\partial{\mathcal{L}}}{\partial{g_{ij}}} =
\begin{cases}
d(f_{\theta}(\vecX_i), f_{\theta}(\vecX_j)) + \tau \ln(g_{ij}) + \alpha_i, & j > i,\\
d(f_{\theta}(\vecX_i), f_{\theta}(\vecX_j)) + \tau \ln(g_{ij}) + \beta_j, & j < i,\\
%\gamma_i, & j = i,\\
\end{cases}
$$

It's obvious that $\frac{\partial^{2}{\mathcal{L}}}{\partial{g_{ij}}\partial{g_{i^{'}j^{'}}}} = 0$ for different pairs $ij$ and $i^{'}j^{'}$. If $ij = i^{'}j^{'}$ then we obtain $\frac{\partial^{2}{\mathcal{L}}}{\partial{g^2_{ij}}} = \frac{\tau}{g_{ij}} > 0$. 

Hence, the Hessian is a positive definite quadratic form. It's important to note that the quadratic form is positive due to the regularizer. Thus, the estimated similarity matrix from Theorem \ref{thm:nearby_log_reg} is a minimum.
\end{proof}

% In a completely similar way we can obtain Theorems for other estimated similarity matrices. Since the pipeline of the proofs is the same as in Theorem~\ref{thm:suf_nearby_log_reg}, then we will omit them to reduce the overall volume of the text.

% For the exponential relation we don't have any regularizer. However non-linear contstraint delivers the minimum.

% \begin{theorem}
% \label{thm:suf_exp}
%  The estimated similarity matrix from Theorem~\ref{thm:exp} is a unique minimum of the optimization problem~\ref{sim:exp}. 
% \end{theorem}

During the proving of the sufficiency condition, we obtained the vital role of the regularizer --- it ensures that our estimated similarity matrix is indeed a unique minimum of the corresponding optimization problem~\eqref{opt_prob}.

\paragraph{Collapse problem}
The analysis of the similarity matrix obtained in  Theorem~\ref{thm:nearby_log_reg} above shows the avoidance of the complete collapse problem. However, such a matrix may still suffer from dimensional collapse, as shown in~\cite{Jing22}.

\subsubsection{Spatio-temporal objective function}
\label{method: objective function}

After defining the ground-truth similarity matrix and the estimated one, we are ready to present our loss function, which is defined for spatio-temporal data. 

Let $X \in \mathbb{R}^{B \times T \times C}$ be an input tensor for the loss function, containing time series for each spatial region, where:
\begin{itemize}
    \item $B$ --- the batch size, the number of different time series.
    \item $T$ --- the history length of each time series.
    \item $C$ --- the embedding size of each time series element.
\end{itemize}

In~\cite{Yue22}, the loss function was divided into \textit{instance} and \textit{temporal} parts. 
We rename the \textit{instance} component to the \textit{spatial} one to highlight the considered dependency type:
$$
\mathcal{L} = \mathcal{L}_{\mathrm{spatial}} + \mathcal{L}_{\mathrm{temporal}}.
$$

\textbf{Spatial loss $\mathcal{L}_{\mathrm{spatial}}$.} Since samples at the same time moments from different time series are semantically independent, we can treat them as images --- make closer positive pairs and far away negative ones. Contrasting performs across the batch dimension. So, the spatial loss is the same as the instance loss in TS2Vec --- standard contrastive objective function based on the similarity matrix~\eqref{eq:sim_orig} applies to the tensor $X[B, None, C]$ at each moment $t \in [1, T]$.

\textbf{Temporal loss $\mathcal{L}_{\mathrm{temporal}}$.} The general form of temporal loss is presented in \eqref{eq:loss_sim_matrix}. To construct an objective function~\eqref{eq:loss_sim_matrix} we need to define both ground truth and estimated similarity matrices. According to the estimated similarity matrix, we use the one obtained from Theorem~\ref{thm:nearby_log_reg}. However, regarding the ground truth, we have two options: the hard dependency relation matrix~\eqref{sim:ma_sim_mat} and the soft dependency one~\eqref{sim:ar_sim_mat}. Thus, we have two different objective functions, resulting in two distinct methods. For brevity, we report only the best results in our experiments, referring to our method as \textbf{DepTS2Vec (Ours)}.
%Indeed, \textbf{SoftTS2Vec (Ours)} is the proposed approach, where the loss function~\eqref{eq:loss_sim_matrix} is formulated using the ground truth similarity matrix from the short-range relation~\eqref{sim:ma_sim_mat} and the estimated similarity matrix provided by Theorem~\ref{thm:nearby_log_reg}. Correspondingly, the \textbf{SoftTS2Vec (AR)} employs the ground truth similarity matrix~\eqref{sim:ar_sim_mat}, while retaining the same estimated similarity matrix as used in SoftTS2Vec (Ours).

To perform temporal contrasting, the objective function was applied to the tensor $X[None, T, C]$ for each region $b \in [1, B]$ in the batch. Following~\cite{Yue22}, we use hierarchical contrasting in both dimensions to capture the local and global contexts of the data. 

\section{Experiments}
\label{practice}

\subsection{Baselines}
We compare our methods with different modern approaches. To make a correct comparison of the \textit{temporal loss}, we use the \textit{spatial loss} described in Section~~\ref{method: objective function} fixed for each of the baselines. Below, we briefly describe how each of them processes temporal dependencies: 

\begin{itemize}
    \item \textbf{TS2Vec}~\cite{Yue22}. In the original version, TS2Vec uses standard ground truth and the estimated similarity matrices for the semantic independent case described in Section~\ref{method:hard_contrastive_learning}.
    \item \textbf{SoftCL}~\cite{lee2024soft}. This method employs a heuristic soft-contrastive similarity matrix, where the closeness between time moments $i$ and $j$ is described by the following formula: $g_{ij} = 2\cdot\sigma(-\tau \cdot |i - j|)$, where $\sigma$ is a sigmoid function, $\tau$ is a hyperparameter.
\end{itemize}

In the subsequent sections, we evaluate the quality of our embeddings across a broad range of downstream tasks. We begin with problems characterized solely by \emph{temporal dependencies} and progressively transition to more complex climatic scenarios within the \emph{spatio-temporal} domain.

\subsection{Temporal downstream problem}

% \subsubsection{Time series classification}

% blablabla

% \subsubsection{Data}
% To assess how well our self-supervised representations capture temporal structure in both univariate and multivariate sequences, we evaluate them on standard time series classification benchmarks from the UCR~\cite{dau2019ucr} and UEA~\cite{bagnall2018uea} archives. The UCR archive contains 128 datasets with univariate time series, whereas the UEA archive includes 30 datasets with multivariate sequences. 
% Together, they cover a broad range of domains such as motion capture, ECG signals, wearable sensors, and spectrographs, and serve as widely adopted benchmarks for time series classification.

% Following the self-supervised pretraining stage described in Section~\ref{exp:architecture}, we freeze the encoder and train a linear classifier on top of the learned representations. 
% For each dataset, we adopt the canonical train/test splits provided in the archives and report classification accuracy on the test set. To handle different sequence lengths and sampling rates, we uniformly resample input sequences to a fixed temporal resolution and normalize each feature independently.

% We evaluate our method against the TS2Vec and SoftCL baselines described above, ensuring that each method uses the same spatial loss component for a fair comparison. 
% Results averaged across multiple datasets are reported in Table~\ref{tab:ucr_uea_results}, demonstrating the ability of our temporal loss to produce meaningful representations that generalize well across both univariate and multivariate classification tasks.

\subsubsection{Downstream problems description}

To assess the effectiveness of our self-supervised representations in capturing temporal patterns, we evaluate the embeddings on a wide range of time series classification tasks. These tasks are drawn from standard, publicly available benchmarks —-- the UCR~\cite{dau2019ucr} and UEA~\cite{bagnall2018uea} archives. Following standard protocol~\cite{10.5555/3454287.3454705, Yue22, lee2024soft}, we extract embeddings from a frozen encoder and train a classifier on top of them. To evaluate the overall performance we measure accuracy and mean rank across datasets.

\subsubsection{Data}

We utilize the well-established time series classification archives from the University of California, Riverside (UCR) and the University of East Anglia (UEA), both commonly used for benchmarking classification algorithms. The UCR archive offers 128 univariate datasets, while the UEA archive includes 30 multivariate datasets. These collections cover a wide range of domains, such as motion capture, ECG signals, wearable sensor data, and spectrographs.

\subsubsection{Downstream pipeline}

After SSL pretraining, we freeze the encoder and extract representations for all sequences. An SVM with RBF kernel is trained on top of these embeddings to perform classification, following the same evaluation protocol as in~\cite{Yue22}.

\subsubsection{SSL architecture}

We adopt the TS2Vec encoder~\cite{Yue22}, a multi-layer temporal convolutional network that generates hierarchical representations. Our temporal loss is applied to the resulting embeddings during self-supervised training, see details in Appendix~\ref{app:implementation_details_temp}.

\subsubsection{Results}

In Table~\ref{tab:ucr_uea_results} we present the classification results on UCR and UEA archives using \emph{average accuracy} and \emph{average rank} metrics. Our proposed method surpasses previous approaches on both UCR and UEA archives. In particular we achieve the highest mean accuracy on both benchmarks -- $84.64$\% on UCR and $72.89$\% on UEA -- outperforming TS2Vec by $2.08$\% and $4.17$\%, and SoftCL by $1.06$\% and $1.20$\%, respectively.

% \begin{table}[ht]
% \centering
% \caption{Comparison of average classification accuracy $(\uparrow)$ and average rank $(\downarrow)$ on UCR and UEA archives. Best results are marked in \textbf{bold}, second one are \underline{underlined}.}
% \label{tab:ucr_uea_results}
% \begin{tabular}{lcccc}
% \hline
% Model & \multicolumn{2}{c}{UCR} & \multicolumn{2}{c}{UEA} \\
% \cline{2-5}
%  & Accuracy & Rank & Accuracy & Rank \\
% \hline
% TS2Vec~\cite{Yue22} & 82.56  & 2.35 & 68.72 & 2.58 \\
% SoftCL~\cite{lee2024soft} & \underline{83.58} & \underline{1.86} & \underline{71.69} & \underline{1.83} \\
% DepTS2Vec (Ours) & \textbf{84.64} & \textbf{1.79} & \textbf{72.89} & \textbf{1.58} \\
% \hline
% \end{tabular}
% \end{table}

\begin{table}[ht]
\centering
\caption{Comparison of average classification accuracy $(\uparrow)$ and average rank $(\downarrow)$ on UCR and UEA archives. Best results are marked in \textbf{bold}, second one are \underline{underlined}.}
\label{tab:ucr_uea_results}
\begin{tabular}{lcccc}
\toprule
\multirow{3}{*}{\textbf{Model}} & \multicolumn{2}{c}{\textbf{UCR}} & \multicolumn{2}{c}{\textbf{UEA}} \\
\cmidrule{2-5}
 & Accuracy & Rank & Accuracy & Rank \\
\midrule
TS2Vec~\cite{Yue22} & 82.56  & 2.35 & 68.72 & 2.58 \\
SoftCL~\cite{lee2024soft} & \underline{83.58} & \underline{1.86} & \underline{71.69} & \underline{1.83} \\
DepTS2Vec (Ours) & \textbf{84.64} & \textbf{1.79} & \textbf{72.89} & \textbf{1.58} \\
\bottomrule
\end{tabular}
\end{table}

% Additionally, we analyze the results across different application domains represented in the UCR and UEA archives. 
% As shown in Table~\ref{tab:ucr_domain_ranks} and Table~\ref{tab:uea_domain_ranks}, our method consistently outperforms TS2Vec in domains such as sensors, devices, electrocardiography (ECG), electroencephalogram (EEG), motion capture, human action recognistion (HAR). 
% We hypothesize that this improvement stems from the fact that in these domains, temporal continuity is a valid assumption: the underlying signals tend to evolve smoothly over time with relatively few abrupt transitions. 
% Since our temporal loss encourages the model to capture gradual changes in latent representations, it is particularly effective in such settings.

To understand the performance of our approach, in Table~\ref{tab:domain_ranks_combined} we compared the mean rank across domains represented in the UCR~\cite{dau2019ucr} and UEA~\cite{bagnall2018uea} time series classification archives. As shown in Table~\ref{tab:domain_ranks_combined}, our method exhibits a clear superiority over TS2Vec with respect to mean rank. The main reason for this result is that all highlighted domains share a common characteristic: they exhibit relatively smooth temporal evolution with few abrupt transitions. This aligns with our theoretical constraint on the data --— the \emph{data continuity} assumption.

\begin{table}[h]
\centering
\caption{Mean rank $(\downarrow)$ by dataset type across methods on UCR and UEA benchmarks.}
\label{tab:domain_ranks_combined}
\begin{tabular}{@{}c|lc|ccc@{}}
\toprule
& \textbf{Type} & \textbf{\# Datasets} & \textbf{TS2Vec} & \textbf{SoftCL} & \textbf{DepTS2Vec} \\
\midrule
\multirow{6}{*}{\rotatebox[origin=c]{90}{\textbf{UCR}}}
& HAR           & 21 & 2.21 & \textbf{1.88} & 1.90 \\
& SENSOR        & 17 & 2.24 & \underline{2.03} & \textbf{1.74} \\
& DEVICE        & 12 & 2.33 & \underline{2.08} & \textbf{1.58} \\
& ECG           & 7  & 2.50 & \underline{2.43} & \textbf{1.07} \\
& MOTION        & 6  & 2.25 & \textbf{1.83} & \underline{1.92} \\
& TRAFFIC       & 2  & \underline{3.00} & \textbf{1.50} & \textbf{1.50} \\
\midrule
\multirow{4}{*}{\rotatebox[origin=c]{90}{\textbf{UEA}}}
& HAR       & 9 & 2.22 & \underline{2.11} & \textbf{1.67} \\
& EEG       & 6 & \underline{2.25} & 2.33 & \textbf{1.42} \\
& MOTION    & 4 & 2.88 & \underline{1.75} & \textbf{1.38} \\
& ECG       & 2 & \underline{3.00} & \textbf{1.50} & \textbf{1.50} \\
\bottomrule
\end{tabular}
\end{table}

\subsubsection{Ablation study}
Table~\ref{tab:domain_ranks_ablation} shows that an optimal similarity measure between samples depends on the temporal characteristics of each dataset type. 
Specifically, HAR datasets are captured most effectively by short-range dependencies, with AR(k=1) for UCR and MA for UEA, reflecting the short-term nature of human activity signals. 
In contrast, SENSOR and DEVICE datasets benefit from long-range context, captured effectively by AR(k=10), while ECG signals are best represented with intermediate horizon, AR(k=5). 
For MOTION, the optimal AR horizon differs across archives: AR(k=5) in UCR and AR(k=1) in UEA.
These findings demonstrate that different temporal structures are best captured by different specifications, providing strong empirical support for the flexibility and generality of our framework.

\begin{table}[h]
\centering
\caption{Mean rank $(\downarrow)$ by dataset type and model specification on the UCR and UEA benchmarks.}
\label{tab:domain_ranks_ablation}
\begin{tabular}{@{}c|l|cccc@{}}
\toprule
& \textbf{Type} & \textbf{MA} & \textbf{AR(k=1)} & \textbf{AR(k=5)} & \textbf{AR(k=10)} \\
\midrule
\multirow{6}{*}{\rotatebox[origin=c]{90}{\textbf{UCR}}}
& HAR           & 3.07 & \textbf{2.00} & \underline{2.21} & 2.71 \\
& SENSOR        & 2.82 & 2.53 & \underline{2.44} & \textbf{2.21} \\
& DEVICE        & 2.83 & 2.92 & \underline{2.29} & \textbf{1.96} \\
& ECG           & \underline{2.29} & 2.64 & \textbf{2.00} & 3.07 \\
& MOTION        &  2.58 &  \underline{2.42} & \textbf{2.00} & 3.00 \\
& TRAFFIC       & \underline{2.25} & \textbf{1.75} & 3.00 & 3.00 \\
\midrule
\multirow{4}{*}{\rotatebox[origin=c]{90}{\textbf{UEA}}}
& HAR           & \textbf{2.11} & \underline{2.22} & 2.50 & 3.17 \\
& EEG           & \textbf{2.33} & \underline{2.50} & 2.58 & 2.58 \\
& MOTION        & 3.13 & \textbf{1.13} & 3.00 & \underline{2.75} \\
& ECG           & 3.50 & 3.50 & \textbf{1.50} & \textbf{1.50} \\
\bottomrule
\end{tabular}
\end{table}

Different types of data contain varying amounts of information. Therefore, it is interesting to determine how much information is required to represent a specific type of data. To assess this across all model configurations, we measured the effective dimensionality by counting the number of principal components needed to capture $95\%$ of the variance~\cite{ansuini2019intrinsic}. We use an embedding size of 320, in accordance with the standard pipeline~\cite{Yue22}. As shown in Table~\ref{tab:effective_dimensionality}, at most only $21\%$ of the embedding dimensions are effectively utilized, indicating a \textit{dimensional collapse}.
This is a known property of contrastive learning methods~\cite{jing2021understanding}, but it also suggests that time series data inherently require fewer latent dimensions for effective representation compared to other domains like computer vision. 
The stability of this metric implies that the choice of temporal dependency model (hard MA vs. soft AR) impacts downstream task performance more significantly than the expressivity of the embedding space.

\begin{table}[h]
\centering
\caption{Mean effective dimensionality of embeddings across dataset types and model specifications.}
\label{tab:effective_dimensionality}
\begin{tabular}{@{}c|l|cccc@{}}
\toprule
& \textbf{Type} & \textbf{MA} & \textbf{AR(k=1)} & \textbf{AR(k=5)} & \textbf{AR(k=10)} \\
\midrule
\multirow{6}{*}{\rotatebox[origin=c]{90}{\textbf{UCR}}}
& HAR           & 58.62 & 58.57 & 57.71 & 57.95\\
& SENSOR        & 65.76 & 65.82 & 65.53 & 65.29\\
& DEVICE        & 38.42 & 38.17 & 38.08 & 37.92\\
& ECG           & 28.14 & 28.57 & 26.29 & 26.43 \\
& MOTION        & 54.33 & 53.83 & 52.50 & 52.33 \\
& TRAFFIC       & 3.50 & 3.50 & 4.00 & 4.00 \\
\midrule
\multirow{4}{*}{\rotatebox[origin=c]{90}{\textbf{UEA}}}
& HAR       & 55.89 & 55.67 & 55.33 & 55.24 \\
& EEG       & 63.67 & 63.17 & 61.67 & 61.67 \\
& MOTION    & 44.25 & 44.25 & 43.75 & 43.75 \\
& ECG       & 2.00  & 2.00  & 2.00  & 2.00 \\
\bottomrule
\end{tabular}
\end{table}

\subsection{Spatio-temporal downstream problems}
Among various climate-related tasks, we selected two extremely essential problems for society: drought prediction and temperature forecasting. Below, we describe the problem statements for both tasks.

\subsubsection{Downstream problems description}
\label{exp:drought_prediction}
\paragraph{Drought prediction}Given the available historical data on the drought index for a particular region, we aim to predict the spatial distribution of drought over the area up to half a year in advance. 
Following~\cite{Mcpherson2022}, we use the Palmer Drought Severity Index, PDSI~\cite{Alley84}, as a forecasting variable. We utilized freely available PDSI data from Google Earth Engine~\cite{gorelick2017google}.

Instead of solving the regression problem, i.e., predicting PDSI directly, we consider a binary classification task: whether it is a drought or not. To convert the regression task to a classification problem, we considered a case as a drought if the PDSI is less than $-2$ following~\cite{Mcpherson2022, Marusov24}.

\paragraph{Temperature forecasting}
The overall problem description is the same as in drought forecasting. In the case of temperature forecasting, we aim to predict a spatial temperature distribution for the next week. 
In particular, we consider a temperature at a height of two meters above the Earth's surface. The data is taken from the public WeatherBench reference dataset ~\cite{rasp2020weather}, which collects historical data on various climate characteristics. Unlike drought prediction, we treat temperature forecasting as a regression task.

\subsubsection{Downstream problem supervised solution pipeline}
The general formalized description of the supervised model based on SSL embeddings for the downstream spatio-temporal problems above is the following:
\begin{itemize}
    \item Input: History of the climatic variable --- \\ 3D-tensor: time, longitude, and latitude.
    \item Output: Predicted distribution of the climatic variable over the given area --- a 2D tensor for a selected forecasting horizon.
\end{itemize}

To solve the target task using the SSL paradigm, we divide our model into two parts: a \textbf{SSL-pretrained encoder} and a \textbf{downstream classifier}. Using the contrastive SSL architecture described in Section~\ref{exp:architecture}, we pretrain the encoder. After SSL pretraining, we freeze the weights and provide feature representations of the input tensors for downstream problems.

\subsubsection{Data}
For the \textbf{self-supervised pretraining}, we use temperature and precipitation data from the WeatherBench dataset covering a period from $1979$ to $2016$ with daily time resolution. To evaluate the performance of the embeddings on the \textbf{downstream problems}, we used:
\begin{itemize}
    \item PDSI index collected for four different regions --- Madhya Pradesh (India), Missouri (USA), Northern Kazakhstan, and Eastern Europe --- to estimate the \textbf{drought prediction} quality. The data covers a time from $1958$ to $2022$ with a monthly resolution. We utilized the TerraClimate Monthly dataset~\footnote{\url{https://developers.google.com/earth-engine/datasets/catalog/IDAHO_EPSCOR_TERRACLIMATE}} to obtain the PDSI data for chosen regions. %Their brief characteristics are summarized in Appendix~\ref{app:data_statistics}.
    \item Temperature data from $2017$ to $2018$ for the \textbf{temperature forecasting} task.
\end{itemize}

\subsubsection{SSL model architecture}
\label{exp:architecture}

% \begin{figure}[h]
%     \centering
%     \includegraphics[width=0.5\textwidth]{figures/ts2vec.pdf}
%     \caption{Original TS2Vec architecture. The figure was taken from~\cite{bazarova2024normalizing}}
%     \label{fig:ST_TS2Vec}
% \end{figure}

 %On the Fig.~\ref{fig:ST_TS2Vec} we show the original architecture of the TS2Vec. 
 Below, we describe the components that we changed in the original TS2Vec to extend it to the spatio-temporal case:

\begin{itemize}
    \item \textbf{Augmentations.} We use standard spatial augmentations from computer vision following~\cite{Wang2022}: Gaussian blur, average pooling.
    \item \textbf{Encoder.} To account for both spatial and temporal dependencies, we use ConvLSTM architecture~\cite{shi2015convolutional}. The output dimension from ConvLSTM is $H \times W \times C_{cell}$ where $H, W$ are spatial sizes of the region and $C_{cell}$ is the embedding size of a single cell. We use one-dimensional convolutions and a fully connected layer to transform the tensor from $H \times W \times C_{cell}$ to $B \times T \times C$ dimensions.
    \item \textbf{Loss function.} Detailed description of our objective function provided in Section~\ref{method: objective function}. 
\end{itemize}

In Appendix~\ref{app:implementation_details_ST} we provide the implementation details for the model. 
% TODO: history length, ...

\subsubsection{Results}

\begin{table}[h]
\caption{ROC-AUC $(\uparrow)$ values for drought prediction for different forecasting horizons and climatic regions}
\centering
\begin{tabular}{llcccccc}
\toprule
\multirow{3}{*}{\textbf{Region}} & \multirow{3}{*}{\textbf{Model}} & \multicolumn{5}{c}{\textbf{Forecasting horizons}} \\ 
\cmidrule{3-7}
& & 1 & 2 & 3 & 4 & 5 \\ \midrule
& TS2Vec~\cite{Yue22} & \textbf{0.62} & 0.56 & 0.53 & \underline{0.52} & 0.51 \\
& SoftCL~\cite{lee2024soft} & 0.59 & \underline{0.58} & \underline{0.55} & \textbf{0.54} & \underline{0.52} \\
{\multirow{-3}{*}{Madhya}} & DepTS2Vec (Ours) & \underline{0.61} & \textbf{0.59} & \textbf{0.57} & \textbf{0.54} & \textbf{0.54} \\
% & TS-BYOL & 0.5 & 0.706 & 0.768 \\ 
% & SN-TS--BYOL & 0.694 & 0.766 & 0.789 \\ 
% & TS2Vec & 0.688 & {\ul 0.788} & 0.816\\
% & SN-TS2Vec, cos & \textbf{0.726} & 0.775 & \textbf{0.852}\\
% {\multirow{-3}{*}{Madhya}} & SoftTS2Vec (AR) &  0.56 &  0.54 & 0.53 & 0.52 & 0.50 \\
\midrule
%\multicolumn{2}{l}{}                                              & 100 & 200 & 400 \\ \hline
& TS2Vec~\cite{Yue22} & 0.71 & \underline{0.68} & 0.65 & 0.61 & 0.58  \\
& SoftCL~\cite{lee2024soft} & \textit{0.75} & \textbf{0.73} & \textbf{0.70} & \textbf{0.67} & \textbf{0.63} \\
{\multirow{-3}{*}{Missouri}} & DepTS2Vec (Ours) & \textbf{0.76} & \textbf{0.73} & \underline{0.69} & \underline{0.65} & \underline{0.62}\\
% & TS-BYOL & 0.5 & 0.796 & 0.933 \\ 
% & SN-TS-BYOL & 0.5 & 0.636 & 0.722 \\ 
% & TS2Vec & {\ul 0.873} & \textbf{0.97} & {\ul 0.952} \\
% & SN-TS2Vec, cos & 0.736 & {\ul 0.909}   & \textbf{1}  \\
% {\multirow{-3}{*}{Missouri}} & SoftTS2Vec (AR) & 0.69 & 0.65 & 0.6 & 0.54 & 0.5 \\ 
%\hline
%\multicolumn{2}{l}{}                                              & 60 & 100 & 200 \\ 
\midrule
& TS2Vec~\cite{Yue22} & 0.62 & 0.57 & 0.54 & 0.52 & 0.5 \\
& SoftCL~\cite{lee2024soft} & \underline{0.69} & \underline{0.65} & \underline{0.61} & \underline{0.58} & \underline{0.54} \\
{\multirow{-3}{*}{Kazakhstan}} & DepTS2Vec (Ours) & \textbf{0.71} & \textbf{0.67} & \textbf{0.63} & \textbf{0.59} & \textbf{0.56} \\
% & TS-BYOL & 0.316 & 0.398 & 0.26 \\ 
% & SN-TS-BYOL & 0.403 & 0.416 & 0.418 \\ 
% & TS2Vec & {\ul 0.476} & {\ul 0.467} & 0.444 \\
% & SN-TS2Vec, cos &  {\ul 0.476} &  0.306 & {\ul 0.663}\\
% {\multirow{-3}{*}{Kazakhstan}} & SoftTS2Vec (AR) & 0.61 & 0.58 & 0.55 & 0.52 & 0.50 \\ 
\midrule
& TS2Vec~\cite{Yue22} & 0.5 & 0.5 & 0.5 & \underline{0.5} & \textbf{0.5} \\
& SoftCL~\cite{lee2024soft} & \textbf{0.64} & \underline{0.6} & \underline{0.54} & \underline{0.5} & \textbf{0.5}\\
{\multirow{-3}{*}{Europe}} & DepTS2Vec (Ours) & \textbf{0.64} & \textbf{0.61} & \textbf{0.55} & \textbf{0.51} & \textbf{0.5} \\
% & TS-BYOL & 0.316 & 0.398 & 0.26 \\ 
% & SN-TS-BYOL & 0.403 & 0.416 & 0.418 \\ 
% & TS2Vec & {\ul 0.476} & {\ul 0.467} & 0.444 \\
% & SN-TS2Vec, cos &  {\ul 0.476} &  0.306 & {\ul 0.663}\\
% {\multirow{-3}{*}{Europe}} & SoftTS2Vec (AR) & \underline{0.53} & 0.5 & 0.5 & \underline{0.5} & \textbf{0.5} \\ 
\midrule
& TS2Vec~\cite{Yue22} & 0.61 & 0.58 & 0.56 & \underline{0.54} & 0.52 \\
& SoftCL~\cite{lee2024soft} & \underline{0.67} & \underline{0.64} & \underline{0.60} & \textbf{0.57} & \underline{0.55}\\
{\multirow{-3}{*}{Aggregated}} & DepTS2Vec (Ours) & \textbf{0.68} & \textbf{0.65} & \textbf{0.61} & \textbf{0.57} & \textbf{0.56} \\
% & TS-BYOL & 0.316 & 0.398 & 0.26 \\ 
% & SN-TS-BYOL & 0.403 & 0.416 & 0.418 \\ 
% & TS2Vec & {\ul 0.476} & {\ul 0.467} & 0.444 \\
% & SN-TS2Vec, cos &  {\ul 0.476} &  0.306 & {\ul 0.663}\\
% {\multirow{-3}{*}{Aggregated}} & SoftTS2Vec (AR) & 0.60 & 0.57 & 0.55 & 0.52 & 0.50 \\ 
\bottomrule
\end{tabular}
\label{table:drought_pred}
\end{table}

% \begin{table}[htb]
%     \centering
%     \begin{tabular}{lcccc}
%      \hline
%      Metric & TS2Vec~\cite{Yue22} & SoftCL~\cite{lee2024soft} & TS2Vec(Ours, MA) & TS2Vec(Ours, AR)\\
%      \hline
%      RMSE & 15.5 & 10.33 & 9.9 & \\
%      \hline
%     \end{tabular}
%     \caption{RMSE ($\downarrow$)}
%     \label{tab:temp_forecast}
% \end{table}

\paragraph{Drought prediction} To evaluate the quality of our embeddings, we use Gradient Boosting as a classifier on top of them to solve the binary classification problem described in Section~\ref{exp:drought_prediction}. As an evaluation metric, we use ROC-AUC. The comparison results for different regions are described in Table~\ref{table:drought_pred}. According to the aggregated statistics across all regions, we can see that our method significantly outperforms the standard TS2Vec approach for all forecasting horizons. According to the comparison with the current state-of-the-art SoftCL approach, we observe a slight improvement for all forecasting horizons. Moreover, our similarity matrices are completely interpretable, while this property is not valid for the SoftCL method.

\paragraph{Temperature forecasting} 
For the temperature forecasting, we used a linear regression based on the representations. Following~\cite{tan2023openstl}, we use a Root Mean Squared Error (RMSE) as an evaluation metric and make a forecast a week in advance. The comparison results are similar to those of the drought prediction and are depicted in Fig.~\ref{fig:temp_forecast}. Indeed, we can see a massive advantage of our method over the standard TS2Vec method and a slight improvement over the SoftCL approach.

\begin{figure}[h]
    \centering
    \includegraphics[width=0.4\textwidth]{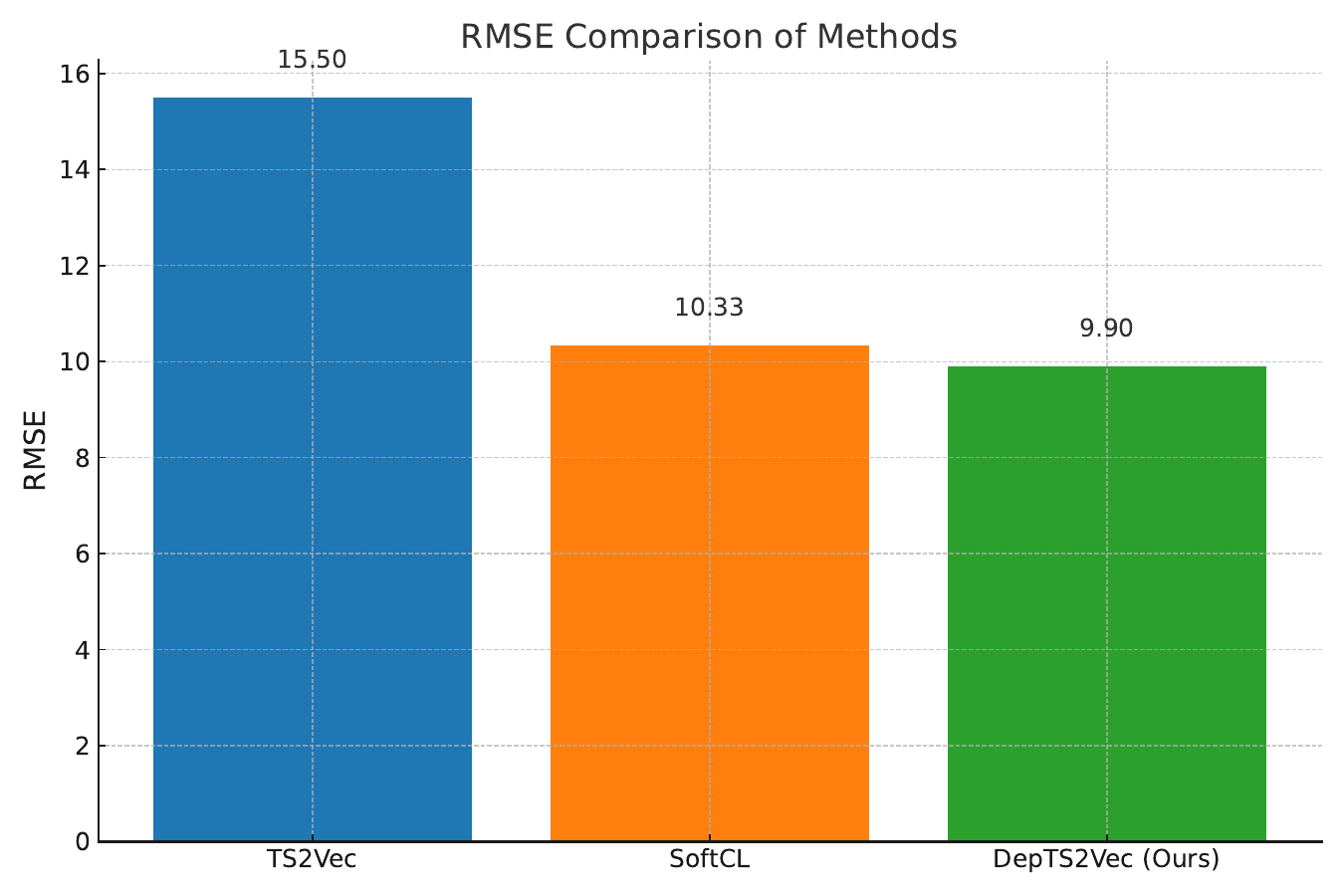}
    \caption{Comparison of the methods via RMSE}
    \label{fig:temp_forecast}
\end{figure}

%\section{Acknowledgments}

%The research was supported by the Russian Science Foundation grant 20-7110135.

\section{Discussion}
The main assumption underlying our theoretical framework is data continuity. Nevertheless, the core idea can be naturally extended to other data domains. For \textit{graph-structured data}, hard dependency corresponds to immediate node adjacency, while soft dependency can be defined through path lengths or spectral distances on the graph Laplacian. For \textit{sequential user data}, recent actions form hard dependencies, whereas older interactions decay in relevance, leading to a soft autoregressive structure. For \textit{non-smooth or seasonal data}, an additional seasonal component can be incorporated into the similarity matrix (analogous to SARIMA~\cite{box2015time} decomposition), capturing periodic dependencies. Our approach can also handle \textit{anomalies} by masking corrupted entries in the similarity matrix, thereby avoiding false positive pairs. Such a mask can be obtained either from external anomaly detection models (e.g., change-point detection) or directly inferred from domain-specific knowledge. Finally, in the presence of \textit{concept drift}, the similarity matrix can be adaptively updated (e.g., in a sliding window fashion), which enables the framework to capture evolving dependency structures in streaming data. In summary, the theoretical foundation we propose is general and provides a principled way to extend self-supervised contrastive learning to a wide class of dependent data types.

\section{Conclusions}

In this work, we introduced a theoretically grounded framework for contrastive self-supervised learning (SSL) specifically designed to handle continuous dependent data. Unlike conventional approaches that assume semantic independence between samples, our method leverages a single data continuity assumption to model both \emph{hard} and \emph{soft} dependencies —-- similar to Moving Average (MA) and Autoregressive (AR) components in ARIMA models—--using analytically derived similarity matrices. Combining different \emph{ground truth} and \emph{estimated} similarity matrices, we construct novel dependency-aware loss functions applicable to both temporal and spatio-temporal domains. 

To evaluate our theoretical framework in practice, we made an extensive experimental evaluation. Our method consistently outperformed state-of-the-art baselines such as TS2Vec and SoftCL. In \emph{time series classification tasks} on the UCR and UEA benchmarks, DepTS2Vec achieved accuracy gains of $2.08$\% and $4.17$\%, respectively, over TS2Vec. For \emph{spatio-temporal downstream problems} such as drought classification, our approach delivered a $7$\% improvement in ROC-AUC, and it also achieved the lowest RMSE in temperature forecasting, outperforming both TS2Vec and SoftCL. 

\section{Acknowledgements}
The work was supported by the grant for research centers in the field of AI provided by the Ministry of Economic Development of the Russian Federation in accordance with the agreement 000000C313925P4F0002 and the agreement with Skoltech №139-10-2025-033.

\bibliographystyle{ieeetr}
\bibliography{bib}

\appendices

\section{Proof of Theorem \ref{thm:nearby_log_reg} from Section~\ref{sec:method}}\label{app:proof_nearby_log_reg}
\begin{proof} 
To solve the constrained optimization problem \eqref{opt_prob}, we will use the method of Lagrange multipliers.  
The Lagrangian is the following for Lagrange multipliers $\boldsymbol{\alpha}, \boldsymbol{\beta}, \boldsymbol{\gamma}$:
\begin{align*}
&\mathcal{L} = \sum_{i \neq j}d(f_{\theta}(\vecX_i), f_{\theta}(\vecX_j)) g_{ij} + \\
&\tau \sum_{i \neq j}  g_{ij}\left(\ln(g_{ij}) - 1 \right) + \sum_{i=1}^{N} \alpha_{i} \left(\sum_{j=i+1}^{N}g_{ij} - 1 \right) + \\ &\sum_{j=1}^{N} \beta_{j} \left(\sum_{i=j+1}^{N}g_{ij} - 1 \right) %+ \sum_{i} \gamma_{i} \mathbf{G}_{i,i},
\end{align*}
%where $\mathbf{W} = \widehat{\mathbf{G}}$.

Calculating the partial derivatives, we get:
$$
\frac{\partial{\mathcal{L}}}{\partial{g_{ij}}} =
\begin{cases}
d(f_{\theta}(\vecX_i), f_{\theta}(\vecX_j)) + \tau \ln(g_{ij}) + \alpha_i, & j > i,\\
d(f_{\theta}(\vecX_i), f_{\theta}(\vecX_j)) + \tau \ln(g_{ij}) + \beta_j, & j < i,\\
%\gamma_i, & j = i,\\
\end{cases}
$$

\begin{align*}
\frac{\partial{\mathcal{L}}}{\partial{\alpha_i}} &= \sum_{j = i+1}^{N}g_{ij} - 1, \\
\frac{\partial{\mathcal{L}}}{\partial{\beta_j}} &= \sum_{i = j+1}^{N}g_{ij} - 1, 
%\frac{\partial{\mathcal{L}}}{\partial{\gamma_i}} = \mathbf{G}_{ii}.
\end{align*}

Setting these derivatives to zero leads to the following system of equations:
$$
j > i: 
\begin{cases}
d(f_{\theta}(\vecX_i), f_{\theta}(\vecX_j)) + \tau \ln(g_{ij}) + \alpha_i = 0\\
\sum_{j = i+1}^{N}g_{ij} - 1 = 0,\\
\end{cases}
$$

$$
j < i: 
\begin{cases}
d(f_{\theta}(\vecX_i), f_{\theta}(\vecX_j)) + \tau \ln(g_{ij}) + \beta_j = 0\\
 \sum_{i = j+1}^{N}g_{ij} - 1 = 0,\\
\end{cases}
$$

% $$
% j = i: \mathbf{G}_{ii} = 0.
% $$

Solving these equations leads to the following results:
$$
g_{ij} = 
\begin{cases}
\frac{e^{\frac{-1}{\tau}d(f_{\theta}(\vecX_i), f_{\theta}(\vecX_j))} }{\sum_{j = i + 1}^{N} e^{\frac{-1}{\tau}d(f_{\theta}(\vecX_{i}), f_{\theta}(\vecX_j))}}, & j > i\\
\frac{e^{\frac{-1}{\tau}d(f_{\theta}(\vecX_i), f_{\theta}(\vecX_j))} }{\sum_{i = j + 1}^{N} e^{\frac{-1}{\tau}d(f_{\theta}(\vecX_{j}), f_{\theta}(\vecX_i))}}, & j < i\\
%0. & j = i
\end{cases}
$$

Rewriting these formulae and accounting the zero diagonal, we obtain the following solution:
$$
g_{ij} = \frac{e^{-\frac{1}{\tau}d(f_{\theta}(\vecX_i), f_{\theta}(\vecX_j))} }{\sum_{k = \min(i,j) + 1}^{N} e^{-\frac{1}{\tau}d(f_{\theta}(\vecX_{\min(i, j)}), f_{\theta}(\vecX_k))}} \mathbbm{1}_{\{i \neq j \}}
$$

\end{proof}

\subsection{DepTS2Vec for temporal data}
\label{app:implementation_details_temp}

For both the general training protocol and the encoder architecture, we follow the default setup from TS2Vec~\cite{Yue22}.

\emph{Loss function.} We use the default instance-wise loss combined with the temporal loss, employing either the MA~\eqref{sim:ma_sim_mat} or AR~\eqref{sim:ar_sim_mat} ground-truth similarity matrices. 
For AR-based similarity, we chose the best hyperparameter $k$ from the range \{1, 5, 10\}. For each dataset in the UCR and UEA archives, we identified distinct optimal hyperparameters, including the choice of ground truth similarity matrix and the hyperparameter $k$. Table~\ref{tab:domain_ranks_ablation} presents the detailed information about the hyperparameters.

%\( k \in \{1, 5, 10\} \).

\subsection{DepTS2Vec for spatio-temporal data}
\label{app:implementation_details_ST}

\emph{General training protocol.} Models are trained for a maximum of 100 epochs. Early stopping halts training if the validation loss fails to decrease for 10 consecutive epochs. We adopt a learning rate of $0.001$, consistent with TS2Vec \cite{Yue22}, and use a history window of 60 time steps. The average training time was 15 hours on a single GeForce GTX $1080$ Ti GPU.

\emph{Encoder architecture.} Each convolutional cell has a hidden dimension and embedding size of 8. Regional embeddings are 512-dimensional. We use three convolutional layers with a kernel size of $3 \times 3$.

\emph{Data augmentations.} 
We use standard computer vision augmentations like Gaussian blur and average-pooling employing a kernel size of $5$.

% \emph{General setting} The training protocol is the following. The maximum number of epochs is $100$. The training stops if, during ten consecutive epochs, the loss function on the validation set doesn't decrease. Learning rate is $0.001$ following TS2Vec~\cite{Yue22}. The history length is $60$.

\emph{Encoder.} The hidden dimension and embedding size for each cell are set to $8$. The embedding size of the region is $512$. Kernel size is $(3, 3)$. The number of layers is $3$.

\emph{Loss function.} We use the default spatial loss from Section~\ref{method: objective function} in combination with a temporal loss, utilizing either the MA~\eqref{sim:ma_sim_mat} or AR~\eqref{sim:ar_sim_mat} ground-truth similarity matrices. Similarly to the temporal case for AR-based similarity we chose the optimal hyperparameter $k$ from the set {1, 5, 10}. According to our experiments MA ground truth similarity matrix~\eqref{sim:ma_sim_mat} showed best performance.

%For our experiments, we use the spatio-temporal objective function described in Section~\ref{method: objective function}. The spatial loss 
%is fixed while for temporal loss we use objective function~\eqref{eq:loss_sim_matrix} with the ground truth MA matrix~\eqref{sim:ma_sim_mat} and the estimated one from Theorem~\ref{thm:nearby_log_reg}.

% \emph{Augmentations} The optimum resize shape for all regions is $60 \times 60$. The kernel size for Gaussian blur and average pooling is $5$. 

\end{document}